\documentclass{article}
\usepackage[final]{ProbNum25}

\usepackage[round]{natbib}

\usepackage{amsmath, amssymb, bm, amsthm, xfrac}
\usepackage{xcolor}
\usepackage{hyperref}
\usepackage[capitalise,nameinlink]{cleveref}
\usepackage{acro,enumitem}
\usepackage{appendix}
\usepackage{arydshln}
\usepackage{colortbl}
\usepackage{tabularx}
\usepackage{graphicx}
\usepackage{subcaption}
\DeclareAcronym{rkhs}{short=RKHS, long=reproducing kernel Hilbert space}
\DeclareAcronym{GP}{short=GP, long=Gaussian process, long-plural=es }
\DeclareAcronym{SDE}{short=SDE, long=stochastic differential equation}
\DeclareAcronym{CG}{short=CG, long= conjugate gradient}
\DeclareAcronym{CompCG}{short=CompCG, long= companion \ac{CG}}
\DeclareAcronym{pCG}{short=pCG, long=preconditioned CG}

\newcommand\numberthis{\addtocounter{equation}{1}\tag{\theequation}}

\newcommand{\reals}{\mathbb{R}}

\newcommand{\set}[1]{\{ #1 \}}
\newcommand{\norm}[1]{\| #1 \|}

\DeclareMathOperator{\spanspace}{span}
\DeclareMathOperator{\matrank}{rank}
\DeclareMathOperator{\matnullity}{nullity}
\DeclareMathOperator{\nullspace}{null}

\algdef{SE}
[STRUCT]
{Struct}
{EndStruct}
[1]
{\textbf{class} \textsc{#1}}
{\textbf{end class}}

\newcommand{\grayline}{\arrayrulecolor{gray!50}\cline{1-12}\arrayrulecolor{black}}

\newtheorem{theorem}{Theorem}
\newtheorem{lemma}[theorem]{Lemma}

\newtheorem{definition}[theorem]{Definition}
\newtheorem{proposition}[theorem]{Proposition}
\probnumtitle{Learning to Solve Related Linear Systems}

\probnumauthors{%
\name{Disha Hegde}%
\affiliation{University of Southampton, United Kingdom}%
\and%
\name{Jon Cockayne}%
\affiliation{University of Southampton, United Kingdom}%
}

\probnumabstract{Solving multiple parametrised related systems is an essential component of many numerical tasks, and learning from the already solved systems will make this process faster. In this work, we propose a novel probabilistic linear solver over the parameter space. This leverages information from the solved linear systems in a regression setting to provide an efficient posterior mean and covariance. We advocate using this as companion regression model for the preconditioned conjugate gradient method, and discuss the favourable properties of the posterior mean and covariance as the initial guess and preconditioner. We also provide several design choices for this companion solver. Numerical experiments showcase the benefits of using our novel solver in a hyperparameter optimisation problem.}

\begin{document}

\section{Introduction} \label{sec:introduction}

We consider the problem of solving multiple related linear systems.
Suppose that these linear systems vary according to some governing parameter $\theta \in \Theta$.
Then the set of problems we are interested in is of the form
\begin{equation} \label{eq:the_system}
A_\theta x_\theta^\star = b_\theta
\end{equation}
where $A_\theta \in \reals^{d \times d}$ is assumed to be symmetric positive definite for each $\theta \in \Theta$, $b_\theta \in \reals^{d}$ is a right-hand-side vector that is given, and the objective is to recover $x_\theta^\star \in \reals^d$.
We will suppose that there is some externally supplied sequence of parameters $(\theta_i)$, $i=1,2,\dots$ and index the systems according to $A_i \equiv A_{\theta_i}$ (similarly $x_i^\star$, $b_i$).
Our task is then, given some information about systems seen for $i=1,\dots,n-1$, to accelerate convergence of an iterative solver for the next system generated by $\theta_{n}$, in an online fashion.

In this work we argue that this problem can be structured as \emph{regression problem} and use the tools of Gaussian process regression to solve it.

\subsection{Related Work} \label{sec:related}

In this work we will focus on iterative linear solvers to solve these related linear systems.
Given the prevalence of related linear systems it is natural that there are many existing approaches that seek to solve this problem. 
These methods can roughly be split into the following three categories:

\paragraph{Initial Guess} Various attempts have been made to find improved initial guesses for the iterative solver in this setting. 
Warm restarting---using the solution of the previous linear system as the initial guess---is a commonly used approach. 
As an extension of this, extrapolating from previous solutions has also been explored, e.g., polynomial extrapolation in \citet{shterev_iterative_2015, austin_initial_2020}, using Taylor expansion in \citet{clemens_extrapolation_2003}, using weighted averages in \citet{ye_improving_2020}. 
Projection methods are also used \citep[e.g.][]{fischer_projection_1998, markovinovic_accelerating_2006} to estimate an efficient initial guess. 
Though these attempts accelerate the linear solver to some extent, using a suitable preconditioner can provide a far greater acceleration. 

\paragraph{Preconditioning}  One efficient way of increasing the efficiency of the linear solver is by preconditioning. 
A commonly used strategy to obtain a good preconditioner for subsequent problems is to start with a general preconditioner and adapt it in some way based on the solution of previous systems \citep[e.g.][]{bergamaschi_survey_2020, li_recycling_2015, carr_preconditioning_2021}. Machine learning based approaches like graph neural networks are commonly used to for this purpose \citep{hausner_neural_2024, li_learning_2023}.

Using covariances related to the problem as a preconditioner was also considered in \citet{calvetti2005priorconditioners}
where a preconditioner for an ill-posed inverse problem is constructed as the covariance matrix of sampled solutions, interpreted as a random variable.
The posterior covariance from a probabilistic linear solver with specific inverse prior has been used to accelerate generalized linear models in \citet{tatzel_accelerating_2023}.

\paragraph{Recycling Subspaces} Another way of accelerating (Krylov-based) linear solvers is to \textit{recycle} the Krylov subspaces explored for previous linear systems. The idea is to improve the search space for the current system by restricting it or enhancing it, leading to faster convergence. This is done by deflation \citep{saad_deflated_2000, de_roos_krylov_2017, daas_recycling_2021}, augmentation \citep{erhel_augmented_2000, carlberg_krylov-subspace_2016}, or a combination of both.

\subsection{Proposed Approach}

We propose to extend the framework of \emph{probabilistic linear solvers} \citep{hennig_probabilistic_2015,cockayne_bayesian_2019} to build a solver for related linear systems.
Existing probabilistic linear solvers (\cref{sec:probabilistic_linear_solvers}) typically function in a Bayesian framework by conditioning a prior distribution on observations of \cref{eq:the_system} (e.g.\ by left-multiplying by some matrix of \emph{search directions} $(S^m)^\top \in \reals^{m \times d}$).
We propose to extend this approach by extending the prior over the parameter space $\Theta$, so that, after conditioning, we can use the posterior to extrapolate to other problems (\cref{sec:methods}).
We refer to this idea as using the probabilistic solver as a \emph{companion regression model} for a classical iterative method.
Specifically, we will use the posterior predictive mean as an initial guess for subsequent systems, and the posterior predictive covariance as the preconditioner for an iterative linear solver.

Compared to the methods described in \cref{sec:related},
our approach is similar to using an extrapolation method to construct an initial guess, and one of the learning-based methods to construct a the preconditioner. 
There is also a peripheral relation to subspace recycling, as we will establish that, to ensure convergence, the initial guess and preconditioner must relate to each other in a particular subspace (see \cref{thm:pCG_spsd_pc}).

\subsection{Structure of the Paper}
The paper proceeds as follows. In \cref{sec:background} we discuss the background on iterative solvers (\cref{sec:iterative_solvers}), probabilistic  linear solvers  (\cref{sec:probabilistic_linear_solvers}) and vector-valued \acp{GP}  (\cref{sec:multi_output}). 
In \cref{sec:methods} we use these tools to extend probabilistic linear solvers across a parameter space, and discuss how to utilise the posterior to accelerate subsequent solves. 
We provide some theoretical results in \cref{sec:theory} and discuss efficient computation strategies in \cref{sec:computation}. 
In \cref{sec:applications}, we perform a simulation study (\cref{sec:sim_study}) to demonstrate the advantages of our method, followed by an application on a \ac{GP} hyperparameter optimisation problem in \cref{sec:optim}.
All the proofs are contained in \cref{sec:proofs}. Some practical implementation details and complexity analysis are discussed in \cref{sec:practical} and \cref{sec:complexity}. Finally, some additional numerical results are presented in \cref{sec:table}.

\section{Background} \label{sec:background}

In this section we will present the required background for the paper.
We will primarily focus on solution of single linear systems, and will suppress dependence of this system on the parameter $\theta$ to simplify the exposition.

\subsection{Iterative Solvers}\label{sec:iterative_solvers}
Iterative linear solvers construct a sequence of iterates $(x^j)$ which are such that $x^j \to x^\star$ as $j \to \infty$.
Well-known examples of such solvers are \emph{stationary iterative methods} such as Richardson's method; in such solvers the iterate $x^{j+1}$ is constructed by applying an affine map to the previous iterate $x^j$.
However such solvers are rarely used in contemporary applications, apart from occasionally to provide preconditioners (see e.g. \citet[Section 5.6]{chen_matrix_2005}); this is due to the fact that the reduction in error $x^j - x^\star$ is typically rather slow, and that convergence is only guaranteed in the limit as $j \to \infty$.

More sophisticated solvers such as Krylov methods \citep[see e.g.][Section 11.3 and Section 11.4]{golub_matrix_2013} are more widely used in practice as a result of their faster convergence.
This faster convergence is obtained because the map applied to $x^j$ to obtain $x^{j+1}$ is typically \emph{nonlinear} rather than affine.
As a result convergence is both faster and, for many methods, guaranteed after a finite number of iterations\footnote{These statements are true only in exact precision; in general convergence in finite precision is often slower}. In this work, we specifically focus on solving symmetric positive definite matrices, and thus restrict ourselves to the \ac{CG} method \citep[Section 11.3]{golub_matrix_2013}.

The convergence of these methods typically depends on two factors:
\begin{enumerate}[nosep]
	\item How far the initial iterate is from the true solution, i.e.\ the initial error $x^0 - x^\star$.
	\item Some measure of the problem's conditioning, often measured as a function of the condition number of the matrix $A_\theta$, or some other matrix related to it.
\end{enumerate}

It is natural to seek to reduce these two quantities, to try and realise an acceleration of the problem.
Improving the problem's conditioning is commonly achieved by building \emph{preconditioners} \citep[Section 11.5]{golub_matrix_2013}.
For a static problem (i.e.\ for a single fixed $\theta$) one is reliant on an understanding of the problem's structure or on generic methods to either improve the initial guess or construct preconditioners.
However, in the present setting where we have access to multiple related linear systems, one can try \emph{learn} good initial guesses and/or good preconditioners based on observations of \cref{eq:the_system} for different values of $\theta$.
In this work we will accomplish using a statistical approach based on \ac{GP} regression \citep{rasmussen_gaussian_2005}.

\subsection{Probabilistic Linear Solvers} \label{sec:probabilistic_linear_solvers}

Recent work has recast some iterative methods as probabilistic numerical methods through probabilistic linear solvers. 
Here we focus on the solution-based Bayesian approach to probabilistic linear solvers as described in \cite{cockayne_bayesian_2019}.
This can be viewed as an iterative method in which the iterates are Bayesian posterior distributions $(\mu^j)$, $j=1,\dots,d$ over $\reals^d$; specifically, these are some prior distribution $\mu_0$ conditioned sequentially on observations of the form 
\begin{equation}\label{eq:bayescg_info}
	(s^j)^\top A x = (s^j)^\top b .
\end{equation}
The vectors $s^j \in \reals^d$ are generally called \emph{search directions}.
These are assumed to be linearly independent and are often also constructed sequentially, by examining the distribution $\mu^{j-1}$ to build $s^j$.
We also introduce the search direction matrices $S^j = [s^1,\dots,s^j]\in \reals^{d \times j}$.

For reasons of conjugacy $\mu^0$ is taken to be Gaussian; $\mu^0 \sim \mathcal{N}(x^0, \Sigma^0)$. 
Using the closed-form expression for a Gaussian posterior under noise-free linear observations such as in \cref{eq:bayescg_info}, one can then show that for all $j=1,\dots,d$ we have the following expressions for the iterates $\mu^j$:
\begin{subequations}
\begin{align}
	\mu^j &= \mathcal{N}(x^j, \Sigma^j) \\
	x^j &= x^0 + \Sigma^0 A^\top S^j (\Lambda^j)^{-1} (S^j)^\top (b - A x_0) \\
	\Sigma^j &= \Sigma^0 - \Sigma^0 A^\top S^j (\Lambda^j)^{-1} (S^j)^\top A \Sigma^0 \\
	\Lambda^j &= (S^j)^\top A \Sigma^0 A^\top S^j.
\end{align} \label{eq:bayescg:posterior}
\end{subequations}

\paragraph{BayesCG}
A prominent choice of search directions in this framework results in the Bayesian conjugate gradient method (BayesCG; \citet{cockayne_bayesian_2019}).
Here the search directions are constructed sequentially via the Lanczos process \citep[Section 10.3]{golub_matrix_2013}, with each successive direction being an orthonormalised version of the previous residual, precisely given by:
\begin{equation} \label{eq:bayescg:sd}
	\tilde{s}^j = r^{j-1} - (r^{j-1})^\top A \Sigma^0 A^\top s^{j-1} \cdot s^{j-1}
\end{equation}
where $\tilde{s}^1 = r^0$, $s^j = \tilde{s}^j / \norm{s^j}_{A\Sigma^0 A^\top}$, and $r^j = b - A x^j$.
One can show that the $s^j$ are orthonormal in the $A\Sigma^0 A^\top$ inner product \citep[Proposition 7]{cockayne_bayesian_2019}.

These directions obtain information about the unknown solution in directions where the residual has largest magnitude, focussing computation on the largest errors in the present estimate.
Two important mathematical consequences of this are that (i) one can show that the error $\norm{x^j - x^\star}$ decays at a geometric rate in $j$, and (ii) the orthogonality properties of the search directions are such that the matrix $\Lambda^j$ is diagonal, which eliminates an $\mathcal{O}(j^3)$ implementation cost.

\paragraph{Prior Choice}
While BayesCG is defined for generic prior covariance $\Sigma^0$, the choice $\Sigma^0 = A^{-1}$ has particular significance, as it results in iterates that coincide with the iterates of CG \citep[see][Proposition 4]{cockayne_bayesian_2019}.
While in general this is an impractical choice of prior, as calculation of the posterior covariance requires that $A^{-1}$ be realised, defeating the point of an iterative solution to the system, in certain applications (\cite{wenger_posterior_2022,Wenger2024CAGPModelSelection,tatzel_accelerating_2023,pfortner_computation-aware_2024,Hegde2025GaussSeidel}) $A^{-1}$ cancels in downstream calculations, making this a more practical choice.
A proxy for $A^{-1}$ considered in \citet{cockayne_bayesian_2019} is to use a preconditioner for $A$. Other works such as \citet{reid_bayescg_2022,Vyas2025Randomised} seek to use the $A^{-1}$ prior implicitly by performing a small number of ``post-iterations'' of CG to compute a low-rank approximation to the posterior.

\subsection{Vector-valued Gaussian processes} \label{sec:multi_output}

The central focus of this work is on extending probabilistic linear solvers, which thus far have only been applied to single linear systems, to the setting of multiple related linear systems. 
The main tool we will use to accomplish this are vector-valued \acp{GP}.
Essentially, we translate from a Gaussian distribution over the solution $x$ of a single linear system to a \ac{GP} over the solution $x_\theta$ as a function of $\theta$.
We first recap the scalar-valued case, before discussing vector-valued \acp{GP}.

\subsubsection{Scalar-Valued \ac{GP} regression} \label{sec:gp_regression_scalar}
\begin{definition}[Gaussian process]
	\citep[Definition 2.1]{alvarez_kernels_2012}
	A random function $x : \Theta \to \reals$ is said to be a \emph{Gaussian process} if, for any finite-dimensional subset $T = \set{\theta_1, \dots, \theta_{n_T}} \subset \Theta$, it holds that $x(T):=[x(\theta_1), \dots, x(\theta_{n_T})]$ follows a multivariate Gaussian distribution. 
\end{definition}

\acp{GP} are typically parameterised by their mean function $m(\theta) = \mathbb{E}(x(\theta))$ and their (positive-definite) covariance kernel $k(\theta, \theta') = \textup{Cov}(x(\theta), x(\theta'))$; we will use the notation $x\sim \mathcal{GP}(m, k)$. 
The primary use case of \acp{GP} is to perform regression. 
Specifically, supposing we are given data $\mathcal{D} := \set{(\theta_i, y_i)}_{i=1}^n$ where $y_i = x(\theta_i) + \zeta_i$, $\zeta_i \stackrel{\textup{IID}}{\sim} \mathcal{N}(0, \sigma^2)$, then the posterior distribution $x \mid \mathcal{D}$ possesses a closed-form given by
\begin{align*}
	x\mid \mathcal{D} &\sim \mathcal{N}(\bar{m}, \bar{k}) \\
	\bar{m}(\theta) &= m(\theta) + k(\theta, T) [k(T, T) + \sigma^2 I]^{-1}(y - m(T)) \\
	\bar{k}(\theta, \theta') &= k(\theta, \theta') - k(\theta, T) [k(T, T) + \sigma^2 I]^{-1}k(T, \theta')
\end{align*}
where $T = \set{\theta_1, \dots, \theta_n}$.

\subsubsection{Vector-Valued Case} \label{sec:gp_regression_vector}
The most straightforward way to extend the above to a vector-valued setting is to think of augmenting the domain of a scalar-valued process with an integer-valued index, representing the index of the (vector-valued) output function.
More formally we construct a Gaussian process on the domain $\Theta \times [1,d]$, where $[1,d]$ is the integers from $1$ to $d$, and abbreviate $x_i(\theta) \equiv x(\theta, i)$. 
We represent the mean as $m_i(\theta) = \mathbb{E}(x_i(\theta)) $ and the (positive definite) kernel as $k_{ij}(\theta, \theta') = \textup{Cov}(x_i(\theta), x_j(\theta'))$.
The marginal distribution for each $\theta$ is therefore multivariate normal, and given by
$$
x(\theta) = \mathcal{N}(m(\theta), C(\theta, \theta))
$$
where $m(\theta) \in \reals^d$ has $i$\textsuperscript{th} component $m_i(\theta)$ while $C(\theta, \theta) \in \reals^{d \times d}$ has $(i,j)$-component $k_{ij}(\theta, \theta)$, with positive definiteness assured by positive definiteness of the covariance kernel.

Under observations of the form $y_i = x(\theta_i) + \zeta_i$, \ac{GP} regression can then be performed by ``stacking'' the covariance matrices $C(\theta, \theta')$ in appropriate blocks, and applying a formula analogous to that given above.
A more explicit description of this can be found in many works, e.g.~\cite{alvarez_kernels_2012}.
In this work, however, as we aim to model the solution vectors $x_\theta$ over the parameter space, we require more general observations of the vector-valued function $x(\theta)$, as described in the next section.

\section{Learning to Solve Related Systems} \label{sec:methods}
We will now extend the methodology in \cref{sec:probabilistic_linear_solvers} to multiple related systems.
To simplify notation we will leave iteration numbers implicit, so that $S_i^j = S_i$ etc.
As described in \cref{sec:probabilistic_linear_solvers}, we will suppose we are given information by projecting the linear systems against a sequence of projection matrices $S_i \in \reals^{d \times m_i}$ with linearly independent columns.
i.e.\ we have
\begin{align} \label{eq:yi}
    S_i^\top A_i x_i^\star = S_i^\top b_i =: y_i. 
\end{align}
To use this information from previously solved systems $i = 1, 2, \dots n$ and the current $n$th system, we augment the dataset with the search direction matrices, i.e.~$\mathcal{D}_n := \{ (\theta_i, S_i, y_i) \}_{i=1}^n$.
The posterior distribution for this choice of data is described in the next section.

\subsection{Full Posterior} \label{sec:posterior_full}
We use a multi-output Gaussian process prior as described in \cref{sec:multi_output}.
Specifically we take $x_\theta \sim \mu_0$, where $\mu_0 = \mathcal{N}(\bar{x}_0, C_0)$, $\bar{x}_0 = \bar{x}_0(\theta)$ and $C_0 : \Theta \times \Theta \to \reals^{d \times d}$ is a positive-definite covariance function.
The posterior distribution is given in the next proposition.

\begin{proposition} \label{prop:posterior}
	We have that $x_\theta \mid \mathcal{D}_n \sim \mathcal{N}(\bar{x}_n, C_n)$, where
\begin{align*}
\bar{x}_n(\theta) &= x_0(\theta) + K_n(\theta) G_n^{-1} z_n \\
C_n = C_n(\theta, \theta) &= C_0(\theta, \theta) - K_n(\theta) G_n^{-1} K_n^\top(\theta)
\end{align*}
and the matrices $K_n(\theta), G_n$ are each block matrices with dimensions $d \times M_n$ and $M_n \times M_n$ respectively, where $M_n = \sum_{i=1}^n m_i$, while $z_n \in \reals^{M_n}$.
Specifically we have
\begin{align*}
    K_n(\theta) &= \left[C_0(\theta, \theta_j) A_j^\top S_j \right]_{1 \leq j \leq n} \\
    G_n &= \left[S_i^\top A_i C_0(\theta_i,\theta_j)A_j^\top S_j \right]_{1 \leq i,j \leq n } \\
    z_n &= \left[y_i - S_i^\top A_i x_0(\theta_i) \right]_{1 \leq i \leq n}
\end{align*}
\end{proposition}
Finally, the cross covariance $\textup{Cov}(x_\theta, x_\theta') = C_n(\theta,\theta')$, where
\begin{align*}
    C_n(\theta,\theta') = C_0(\theta, \theta') - K_n(\theta) G_n^{-1} K_n^\top(\theta').
\end{align*}

The next result guarantees invertibility of $G_n$ provided the $S_i$ have linearly independent columns.

\begin{lemma} \label{lemma:invertible}
	Assume $A_i$ is invertible and $S_i$ has linearly independent columns for $i = 1,\dots, n$. Then $G_n$ is invertible.
\end{lemma}

Although we focus on symmetric positive definite matrices in this work, the posterior  given in \cref{prop:posterior} holds more generally and, in particular, is well defined if $A_i$ is nonsingular for $i=1,\dots,n$.
It could therefore be treated as a probabilistic linear solver in its own right, though this is not the focus of this paper.

Having defined the posterior distribution we now turn to how it can be used in practice.
As mentioned in \cref{sec:introduction}, we propose to use the posterior as a companion regression model, and use it to construct both initial guesses and preconditioners that can be passed to a classical iterative linear solver.
In the next section we will present some theoretical results that justify this approach.

\section{Theoretical Results} \label{sec:theory}
We now present some theoretical properties of the posterior covariance of the companion model. 
Our first result shows that the posterior predictive covariance has full rank (and is thus positive definite), when evaluated at locations not contained in the training set.
\begin{theorem} \label{thm:rank(C_n-1)}
	Assume the covariance $C_0$ is positive-definite.
    Then $C_{n-1}(\theta_n,\theta_n)$ is full rank for all $\theta \not\in T$. 
\end{theorem}

The next result shows that when evaluated at the last point contained in the training set, the posterior covariance is rank-deficient.
Note that this holds more generally for any point in the training set (i.e.\ $C_n(\theta_i, \theta_i)$ is rank deficient for any $i=1,\dots,n$). 
However only the last point is relevant for subsequent developments.
\begin{theorem} \label{thm:rank(C_n)}
	The rank of $C_{n}(\theta_n,\theta_n)$ is $d - m_n$. Its null space is spanned by the columns of $A^\top_n S_n $.
\end{theorem}

The above theorem explicitly provides the null space of the posterior covariance. This motivates our next result, which states that the posterior mean from the companion regression method is equal to the actual solution of the linear system, in the null space of the posterior covariance.
\begin{theorem} \label{thm:nullC_n}
    In the null space of $C_n(\theta_n, \theta_n)$ it holds that $\bar{x}_n(\theta_n) = x^\star_n$, i.e. 
    $$S_n^\top A_n x_n(\theta_n) = S_n^\top A_n x^\star_n.$$
\end{theorem}

This result provides us the motivation to use the combination of the posterior predictive mean and covariance as initial guess and the preconditioner for \ac{pCG} (see \cref{alg:pCG}).
As the posterior mean is already correct in the null space of the posterior covariance, using the posterior covariance as the preconditioner allows the fast exploration of the unexplored directions where the solution is yet to be identified. 
In fact, the next theorem shows that this combination is essential to use the posterior covariance as preconditioner for \ac{pCG}. 

While \ac{pCG} has been shown to be well-defined for positive semi-definite matrices \citep{kaasschieter_preconditioned_1988,hayami_convergence_2020}, and for projection-type rank-deficient preconditioners \citep{frank_construction_2001}, using a positive semi-definite preconditioner that does not commute with $A$ has not been explored in the literature, to the best of our knowledge. 
In the following result we prove that when a singular preconditioner is used in \ac{pCG}, the algorithm still converges to the truth provided that the initial guess satisfies a specific condition.

\begin{theorem} \label{thm:pCG_spsd_pc}
    Given a symmetric positive semi-definite preconditioner $P$, let $U_1$ denote a matrix whose columns form an orthonormal basis of the range of $P$ and $U_2$ a matrix whose columns form an orthonormal basis of its null space.
    Consider solving the linear system  $A x^\star = b$ using pCG from \cref{alg:pCG}.
    We have that if $U_2^\top x_0 = U_2^\top x^\star$, then pCG with singular preconditioner $P$ has iterates
    $$x_k = U_1\tilde{x}_k + U_2 U_2^\top x^\star$$
    where $\tilde{x}_k$ are iterates from applying pCG to a modified system $U_1^\top A U_1 x^\star = \bar{b}$, (where $\bar{b} = U_1^\top b - U_1^\top A U_2 U_2^\top x_0$),
    with positive-definite preconditioner $U_1^\top P U_1$.
\end{theorem}

The theorem stated above has some relation to \emph{projected CG} \citep{Zheng2020}, which considers using singular preconditioners within nonlinear CG methods for convex optimisation to constraints that ensure the iterate remains in the feasible set.
Our focus on the linear setting allows us to establish stronger results than can be applied in the more general nonlinear CG setting.

As  \cref{thm:nullC_n} guarantees that our posterior mean satisfies this condition, we now have the liberty to use the posterior from the companion regression model to accelerate the conjugate gradient solver. 
In the next section we discuss some practical details.

\section{Modelling Choices} \label{sec:computation}
We now consider the critical problem of choosing the prior and search-directions.
Together these completely specify the companion regression model for the iterative solver.

\subsection{Choices of Prior}

We mostly ignore the task of choosing the prior mean.
While this will clearly have a strong influence on the effectiveness of the initial guess provided by the companion regression model, it is generally problem-specific.
Instead, we focus on the prior covariance, for which several natural choices present themselves.

\paragraph{Tensor Product Prior}\label{sec:choices_prior}
The tensor product prior is perhaps the most common choice for multi-output GP regression \citep[e.g.][Section 4]{alvarez_kernels_2012}.
With this prior we take $C_0(\theta, \theta') = k(\theta, \theta') \Sigma$, where $\Sigma$ is a fixed symmetric positive-definite $d \times d$ matrix, while $k(\theta, \theta')$ is a scalar-valued positive-definite kernel function.
The $(i,j)$ block of $G_n$ can therefore be simplified to $G_{n,(i,j)} = k(\theta_i, \theta_j) \times S_i^\top A_i \Sigma A_i^\top S_j$, so that $G_n$ can be factorised as
\begin{align*}
G_n =  diag \left[S_i^\top A_i\right]_{1 \leq i \leq n} 
\left[\Sigma \otimes k(T,T) \right] 
 diag \left[A_i^\top S_i\right]_{1 \leq i \leq n}
\end{align*}
where $k(T,T) \in \reals^{n \times n} $ is the kernel function on $T = \left(\theta_1, \theta_2, \dots \theta_n\right)$.

Unlike in standard multi-output regression problems, there is no particular structure that can be exploited in $G_n$ (i.e.\ $G_n$ does not have the form of a Kronecker product, whose inverse can be expressed as the product of the two matrix inverses).
This is because the information used is heterogeneous apart from in the special case $S_i^\top A_i$ is independent of $i$, e.g.\ if $S_i = A_i^{-1}$.

\paragraph{A Natural, Nonseparable Prior Covariance}

Motivated by analogy with the literature on probabilistic linear solvers, we propose an (impractical) \emph{nonseparable} prior covariance.
When $A_\theta$ is symmetric positive definite for all $\theta$, the prior $C_0(\theta, \theta) = A_\theta^{-1}$ has a special status as discussed in \cref{sec:probabilistic_linear_solvers}.
Thus, taking $L_\theta$ to be a factorisation of $A_\theta$ such that $A_\theta = L_\theta L_\theta^\top$, this motivates the cross covariance $C_0(\theta, \theta') = L_\theta^{-1} k(\theta, \theta') L_{\theta'}^{-T}$. 
Thus, we consider nonseparable priors of the form
\begin{align*}
C_0(\theta, \theta) = k(\theta, \theta) A_\theta^{-1} \qquad C_0(\theta, \theta') =  L_\theta^{-1} k(\theta, \theta') L_{\theta'}^{-1}
\end{align*}
where $k(\theta, \theta')$ is again a scalar-valued positive-definite kernel function.

Under this choice of prior note that we can simplify all but one of the matrices involved in computation of the posterior to remove explicit dependence on $A_i^{-1}$.
For example, we have that 
$$
G_n = \begin{pmatrix}
	S_1^\top A_1 S_1 & S_1^\top L_1 L_2^\top S_2 & \dots & S_1^\top L_1 L_n^\top S_n \\
	S_2^\top L_2 L_1^\top S_1 & S_2^\top A_2 S_2 & \dots & S_2^\top L_2 L_n^\top S_n \\
	\vdots & \vdots & \ddots & \vdots \\
	S_n^\top L_n L_1^\top S_1 & S_n^\top L_n L_2^\top S_2 & \cdots & S_n^\top A_n S_n
\end{pmatrix}.
$$
However, note that the above still depends on the factors $L_i$,
computation of which is still associated with a cubic cost, rendering this impractical.
Moreover the cancellation does not mitigate dependence of the posterior covariance $C_n(\theta, \theta)$ on $A_\theta^{-1}$, since it still has the form of a downdate of $A_\theta^{-1}$.
However we note that in several recent works \citep{wenger_posterior_2022} it has been shown that one can use posteriors of this form in downstream applications in such a way that $A_\theta^{-1}$ cancels.
We do not exploit this fact in this paper, but observe this only to highlight that this prior choice is also of practical, not just theoretical interest. 

Note that this choice of prior is still not completely general as it assumes that smoothness as a function of $\theta$ is identical for all components of the solution vector.
However, nothing in the methodology presented prevents using general prior covariances.

\subsection{Choices of Search Directions} \label{sec:choices_sd}

Thus far we have made no assumptions on $S_1, \dots, S_n$ apart from that they are linearly independent.
We now consider some natural choices.

\paragraph{Observation of the Solution}
Assuming we are willing to compute $x^\star_1, \dots, x^\star_n$, for the training set, it makes sense to consider the choice $S_i = A_i^{-1}$, corresponding to direct observation of the solution for each $\theta_i$. This simplifies the posterior mean and covariance to be:
\begin{align*}
	\bar{x}_n(\theta) &= x_0(\theta) + C_0(\theta, T) C_0(T, T)^{-1}  X^\star_n
	\\ C_n(\theta, \theta') &= C_0(\theta, \theta') - C_0(\theta, T) C_0(T, T)^{-1} C_0(\theta', T)^\top
\end{align*}
where $C_0(\theta, T) \in \reals^{d \times dn}$ is a block matrix with elements $C(\theta, \theta_i)$, $j = 1, 2, \dots n$, $C_0(T, T) \in \reals^{dn \times dn}$ is a black matrix with elements $C(\theta_i, \theta_j)$, $i, j = 1, 2, \dots n$; and $X^\star_n \in \reals^{dn}$ is the stacked form of solutions of linear systems $x^{\star}_i, i = 1, 2, \dots n$. 
 
This is equivalent to performing a multi-output Gaussian process regression with the parameter $\theta$ as the input variable and the solution $x^\star_{\theta}$ as the output variable, with $C_0$ as the prior kernel.

However, this choice compels us to update the companion regression model after the linear system has been solved, utilising no information about the current linear system being solved to condition on, while computing its initial guess and the preconditioner. We now propose choices of search directions that can more actively use the information from the linear system being solved.

\paragraph{Subset of Data}
Related to the above, we could also take $S_i$ to be given by observing a subset of coordinates of $b_i$.
Let $\mathbb{I}_n \subset 1\colon d$; we then take $S_i = I_{:,\mathbb{I}_{m_i}}$, i.e.\ the identity matrix with a subset of columns defined by $\mathbb{I}_{m_i}$.
Attention then turns to how we should select $\mathbb{I}_{m_i}$.
We take an approach similar to the one in \cite{Cockayne2021}. 
We select the coordinates of $b_i$ by minimising a heuristic based on fill distance between the new selected co-ordinates and the old co-ordinates in an augmented space of  co-ordinates and the parameters of the linear system. 

\paragraph{Krylov-Based Directions}
We can view the predictive distribution of $\mu_n$ at $\theta_{n+1}$ as providing a prior distribution for a probabilistic linear solver at that location.
For convenience let $\mu_{n,n+1}$ denote this predictive distribution.
We then have $$\mu_{n,n+1} \sim \mathcal{N}(\bar{x}_{n,n+1}, C_{n, n+1}),$$ where $\bar{x}_{n,n+1} = \bar{x}_n(\theta_{n+1})$ and $C_{n,n+1} = C_n(\theta_{n+1}, \theta_{n+1})$.

An appealing choice of directions would then be the ``adaptive'' BayesCG search directions.
These are obtained by running BayesCG with prior given by $\mu_{n,n+1}$, so that the prior mean / initial iterate is $\bar{x}_{n,n+1}$ and the prior covariance is $C_{n,n+1}$. Using \cref{eq:bayescg:sd}, this leads to search directions of the form
\begin{align*}
s_ i &= r_{i-1} - r_{i-1} A_{n+1} C_0(\theta_{n+1},\theta_{n+1}) A_{n+1}^\top s_{i-1} \cdot s_{i-1} \\ &- r_{i-1} A_{n+1}  K_n(\theta_{n+1}) G_n^{-1} K_n^\top(\theta_{n+1}) A_{n+1}^\top s_{i-1} \cdot s_{i-1}.
\end{align*}

\subsection{Practical Implementation}
There are few more nuances to consider while implementing this practically. We present the conjugate gradient (\textsc{pCG}) and algorithms for updating the companion regression model (\textsc{CompCG}) algorithm in \cref{alg:pCG} and \cref{alg:CompCG}. Optimising hyperparameters for the regression model forms an important part our algorithm. While we do not identify a single best method to do this, we discuss our approach as well as efficient implementation employing Cholesky factor updating in \cref{sec:practical}. 
Limiting $M_n$---the size of $G_n$---is important to keep our method computationally feasible, and we present a few ways to ensure this by limiting the size of the search direction matrices, limiting the frequency of updates and truncation of information in \cref{sec:train_size}. 
We also analyse the computational complexity of this companion regression model in \cref{sec:complexity}.

\section{Experiments} \label{sec:applications}
We now turn to empirical evaluation of the novel solver.
In \cref{sec:sim_study} we perform a simulation study in an idealised setting, while in \cref{sec:optim} we consider application to a more realistic GP hyperparameter optimisation problem. A python implementation of this solver is available on \texttt{GitHub}\footnote{\url{https://github.com/hegdedisha/learning_related_systems}}.

\subsection{Simulation Study} \label{sec:sim_study}

\paragraph{Experimental Setup}

For this simulation study, we aim to evaluate performance in the correctly specified setting. 
To simulate an optimisation problem let $\Theta = \reals^{d'}$, $d'=200$. 
We take $\theta_i = \theta_{i-1} + \frac{0.05\epsilon_i}{i}$, where $\epsilon_i \in \reals^{d'}$ with $j$\textsuperscript{th} component  $\epsilon_{ij}\sim \mathcal{U}(0,1)$ ensuring that the parameters become closer together for larger $i$. 
We then sample $\{x_{\theta_i}\}_{i=1}^{n}$ jointly from the prior. 
The prior mean was set to zero, and we used the tensor product prior covariance as described in \cref{sec:choices_prior}, with $\Sigma = I_d$ and $k$ as Matèrn $\frac{3}{2}$ kernel.
The matrix $A_\theta$ of dimension $d=500$ is defined as $A_{\theta} = U_1 \Lambda_{\theta} U_1^\top $, where $U_1$ is an orthogonal matrix drawn from the Haar measure (fixed for all $\theta$), and $\Lambda_{\theta} = \textup{diag}([\theta; U])$ where $U \in \reals^{d-d'}$ with $U_{j} \sim \mathcal{U}(0.8,100)$ (again fixed for all $\theta$).
We then compute the right-hand side from $b_\theta = A_\theta x_\theta$.  To explore the complete potential of our method, we do not employ any cautious updating strategies (\cref{sec:train_size}) in this section. We run the \ac{CG} iterations for all the methods up to the default relative tolerance of $10^{-5}$.

\begin{figure*}
    \includegraphics[width= 2\columnwidth]{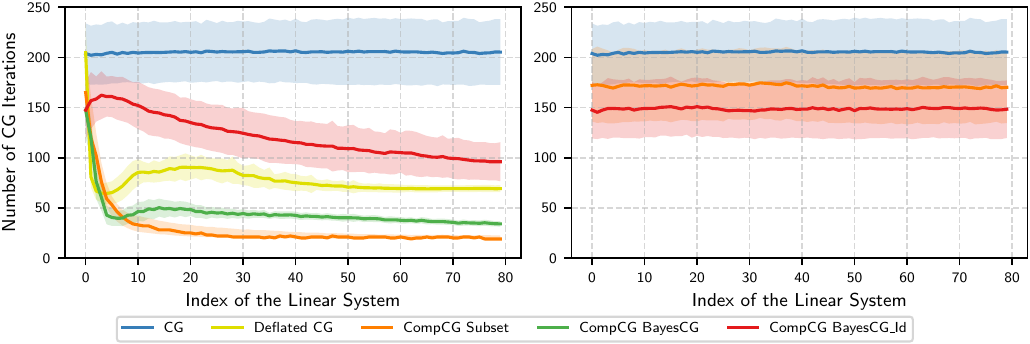}
	\begin{minipage}[t]{.48\linewidth}
	\centering
	\subcaption{Number of Conjugate Gradient Iterations for different choices of search directions, compared to CG and deflated CG.}\label{fig:cg_iters}
	\end{minipage}%
	\hspace{0.04\textwidth}
	\begin{minipage}[t]{.48\linewidth}
	\centering
	\subcaption{Number of Conjugate Gradient Iterations for different choices of search directions without learning, compared to CG.}\label{fig:cg_iters_nl}
	\end{minipage}
    \caption{Number of Conjugate Gradient Iterations for different choices of search directions, compared to CG.} 
\end{figure*}

We first evaluate the performance of different choices of priors and search directions. \cref{fig:cg_iters} shows the number of CG evaluations required for each linear system for each of the methods, averaged over 100 runs, with 1 standard deviation as the shaded region. Subset and BayesCG refer to the search directions described in \cref{sec:choices_sd}, while BayesCG\_Id refers to the BayesCG search directions computed using zero prior mean and identity prior covariance, as a less computationally expensive varient of BayesCG. We also plot an additional baseline of deflated CG, a recycled deflation-based preconditioner obtained using 
$m$ Ritz vectors (as described in \citet{gaul_recycling_2014}, implemented from the Python package \texttt{krypy}) as a method that adapts according to the previously solved systems. 

We see that all three \ac{CompCG} variants take a fewer CG iterations to reach the required error tolerance. As the BayesCG search directions are more informative, CompCG BayesCG converges faster than BayesCG\_Id; this is because the former uses directions that are adapted to the predictive covariance, i.e.\ that are aware of what directions have already been explored by the regression model, while the latter are ignorant of this.
The subset search directions outperform the BayesCG search directions, however, representing a lower cost but computationally competitive alternative. We also see that the deflation based preconditioner marginally outperforms CompCG for initial linear systems, but as the model is trained more, CompCG (both subset and BayesCG variants) overtakes deflation preconditioner.

\begin{figure}
		\includegraphics[width= \columnwidth]{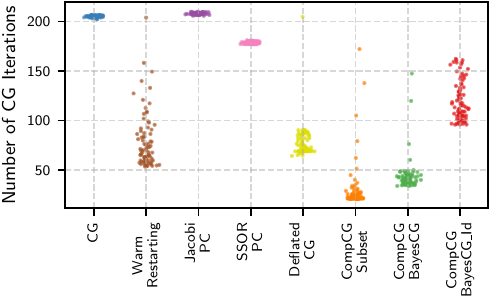}
    	\caption{Comparison of number of conjugate gradient iterations for different preconditioners.} \label{fig:cg_iters_pc}
\end{figure}

We also examine the effect of the global prior model, as opposed to the preconditioning effect by solely performing pre-iterations with search directions to determine an initial guess and preconditioner.
\cref{fig:cg_iters_nl} contrasts \cref{fig:cg_iters} by ignoring the regression model; that is, it assumes that the $x_\theta$ are independent across $\theta$.
Note that CompCG BayesCG is not displayed here, as the predictive distribution in this setting is simply the prior.
We can see that this performs significantly worse than \ac{CompCG} in \cref{fig:cg_iters}, suggesting that the regression model provides value over-and-above the preconditioning effect from the pre-iterations, i.e.\ some \emph{global} preconditioning effect has been learned from the data \emph{in addition to} the preconditioning effect from pre-iterations.
We also note that in this case BayesCG\_Id search directions outperform the Subset search directions. 
This is due to the fact that the Subset search directions borrows its strength from the learned information.

In \cref{fig:cg_iters_pc}, we compare the performance of the CompCG method (for different search directions) with linear solvers that are routinely used in practice:
\begin{enumerate}[nosep]
\item CG,
\item CG with warm-restarting,
\item Preconditioned CG with:
\begin{enumerate}
    \item Jacobi preconditioner,
    \item SSOR preconditioner with $\omega = 1$, the Gauss-Seidel preconditioner.
\end{enumerate}
\item Deflated CG,

\end{enumerate}
with each point indicating the total number of CG iterations across all linear systems. For further details on preconditioners used see \citet[Section 11.5.3]{golub_matrix_2013}. We notice that all three variants of \ac{CompCG} outperform both  warm restarting and the Jacobi preconditioner, while CompCG Subset and BayesCG also outperform the SSOR preconditioner and deflated CG.

\subsection{GP Hyperparameter Optimisation Problem} \label{sec:optim}

\begin{table*}

	\begin{tabular*}{\columnwidth}{l|lllllllllll}
		$d$ & 162  & 648 & 800 & 1035 & 1352 & 1800 & 2592 & 4050 & 7200 & 10368 & 16200 \\

        \grayline
		\multicolumn{11}{l}{\rule{0pt}{1em}Time for solving the linear systems (sec)} \\
		\grayline
		\rule{0pt}{1em}CG & 0.25  & 0.85 & 0.99 & 1.55 & 1.32 & 3.13 & 7.09 & 73.31 & 193.94 & 347.00 & 1671.25 \\
		\begin{tabular}{@{}l@{}}CompCG Subset\end{tabular}  & \textbf{0.21}  & \textbf{0.43} & 0.51 & \textbf{0.79} & 1.89 & 1.47 & 3.09 & 14.19 & \textbf{48.09}  & 131.81 & \textbf{177.77}  \\
		\begin{tabular}{@{}l@{}}CompCG BayesCG \end{tabular}  & 0.24 & 0.53 & \textbf{0.48} & 0.79 & \textbf{1.20} & \textbf{1.42} & \textbf{2.51} & \textbf{13.22} & 59.34  & \textbf{ 95.54}  & 272.27  \\
        \grayline
		\multicolumn{11}{l}{\rule{0pt}{1em}Average time for solving linear systems per optimiser iteration (sec)}  \\
		\grayline
		\rule{0pt}{1em}CG & 0.0031 & 0.0112 & 0.0118 & 0.0185 & 0.0220 & 0.0373 & 0.0844 & 0.8728 & 2.8521 & 6.6730 & 19.8959 \\
		\begin{tabular}{@{}l@{}}CompCG Subset\end{tabular} &  
        0.0035 & 0.0083 & 0.0106 & \textbf{0.0116} & 0.0224 & 0.0263 & 0.0514 & \textbf{0.2729} & \textbf{0.7514} & 1.6476 & 4.0402  \\
		\begin{tabular}{@{}l@{}}CompCG BayesCG \end{tabular}   & 
        \textbf{0.0029} & \textbf{0.0069} & \textbf{0.0086} & 0.0142 & \textbf{0.0215} & \textbf{0.0253} & \textbf{0.0419} & 0.2754 & 0.8726 & \textbf{1.4929} & \textbf{3.7815}  \\
		\grayline
		\multicolumn{11}{l}{\rule{0pt}{1em}Wall-time (sec)} \\
		\grayline
		\rule{0pt}{1em}CG & 0.82 & 4.78 & 4.99 & 7.50 & \textbf{8.75}  & 24.38 & 51.84 & 407.08 & 1475.93 & 4554.74 & 11241.33 \\   
		\begin{tabular}{@{}l@{}}CompCG Subset \end{tabular} & \textbf{0.62 }& \textbf{2.80 }& \textbf{3.41} & 6.46 & 10.86 & \textbf{13.79} & 59.34 & 404.10 & \textbf{1202.75} & \textbf{2963.18} & 8440.28  \\
		\begin{tabular}{@{}l@{}}CompCG BayesCG \end{tabular} &  0.81 & 2.90 & 3.84 & \textbf{6.09}& 14.50 & 23.33 & \textbf{49.27} & \textbf{297.70} & 1955.75 & 3135.40 & \textbf{7469.10}  \\

	\end{tabular*}
	\caption{Results of NLL optimisation for the regression problem in Section 6.2 for varying $d$.}\label{tbl:nll_times}
\end{table*}

We now consider a more applied setting in which related linear systems arise. 
We consider a \ac{GP} regression problem on the ERA5 global 2m temperature dataset \citep{era5_2023}, for a single timestamp on 1st January 2024. 
To perform this regression problem one must optimise the hyperparameters of the GP.
The standard approach to doing so is by minimising the log-marginal likelihood given by,
\begin{align*}
	\log p(y|T, \phi) &= -\frac{1}{2}y^\top k_\phi(T, T)^{-1} y - \frac{1}{2}\log\det k_\phi(T,T) \\
	&-\frac{n}{2}\text{log}(2\pi),
\end{align*}
\sloppy results in a sequence of linear systems of the form $k_\phi(T, T)z = y$ that must be solved (see \citet[Chapter 5]{rasmussen_gaussian_2005}).
We use a  Matèrn $\sfrac{3}{2}$ kernel function, setting the kernel hyperparameters to be the length-scale, amplitude and noise variance (collectively denoted by $\phi$). 
These define the parameter space $\Theta$ for the sequence of regression problems.
We minimise the negative log-likelihood using the \texttt{scipy} \citep{2020SciPy-NMeth} implementation of L-BFGS (as described in \cite{byrd1995limited,zhu1997algorithm}). 
The log determinant is approximated using approximation methods from \citet{wenger2022preconditioning}, to avoid the cubic cost this would otherwise represent.


This dataset has a spatial resolution of approximately 31km, which results in a grid of approximately 1 million points. We consider subsets of this grid with varying resolution, leading to regression training sets of differing size. Unlike \cref{sec:sim_study}, we build a more conservative companion regression model here using measures explained in \cref{sec:train_size}. We update the model if the number of CG iterations is high, and reset the model if we see a significant jump in the NLL (indicated by very high number of CG iterations) indicating that a new computational regime has been encountered. We also restrict the size of the training dataset to 4, to ensure fast inversion of $G_n$. 
\cref{tbl:nll_times} shows that time required for the optimiser to converge for each of CG, CompCG Subset and CompCG BayesCG, for varying sizes $d$ of the linear systems to be solved (we do not include CompCG BayesCG\_Id here as it shows poor performance for the experiments in \cref{sec:app_simulation}). A more detailed version of this table is presented in the appendix as \cref{tbl:app_nll_times}. All the computations were performed on University of Southampton's Iridris5 cluster on a single node with 40 CPU cores.

Both versions of \ac{CompCG} outperform CG in terms of total time required for solving the linear systems and average time for solving linear systems per optimiser iteration for the all the cases, with the difference generally increasing as the $d$ increases. However, this dominance is not clearly reflected in wall-times. We speculate this is due to a combination of the optimiser taking a suboptimal path through the parameter space before reaching the minimum NLL and the randomness in the calculation of the determinant.
This is a pathology of the integration of \ac{CompCG} with the optimiser that we will aim to study in future work, and could perhaps be addressed by better log determinant approximation routines and more careful optimiser choice.

\section{Conclusion} \label{sec:conclusions}
We introduced a companion regression model for the conjugate gradient linear solver that provides an efficient initial guess and preconditioner pair for fast solution of related linear systems. We analysed the properties of this choice theoretically, and we highlighted through our simulations that this can drastically decrease the number of iterations of the linear solver and improve the computation time in problems that employ linear solvers in a loop, such as hyperparameter optimisation problems.

The primary limitation of this companion regression model is the additional time taken for the computation of the  regression model. Even though the efficiency of the initial guess and the preconditioner obtained reduces the number of conjugate gradient iterations required for the traditional solver, this reduction in time does not always make up for the additional time required to maintain the regression model. It should be noted that the regression model here is implemented Python; more efficient low-level implementation would likely address this in part.

Along with the more efficient implementation, there are several interesting opportunities for future work. Although we empirically see that conditioning on previously solved system to obtain the preconditioner improves the performance (\cref{fig:cg_iters,fig:cg_iters_nl}), our theory in \cref{sec:theory} about recycling subspaces does not explicitly provide us the motivation to use the learned data. We plan to explore this phenomenon more theoretically in our future work. We do not delve into tuning the hyperparameter for the regression model here and it would be interesting to explore more sophisticated hyperparameter optimisation techniques. We use a general separable covariance function for the regression model here, but it would be useful to exploit a problem-specific covariance function which would make the regression model more efficient. Another interesting approach would be to utilise thinned search directions to update the regression model, as \cref{fig:cg_iters_ns} indicates some search directions might be more informative than others. Finally, we focus on symmetric positive definite matrices and the \ac{CG} solver in the work. But as mentioned before, our regression model for the companion solver only requires invertible linear systems, and thus we plan to extend this companion solver framework to more general solvers in the future.

\section*{Acknowledgements}

JC was supported by EPSRC grant EP/FY001028/F1.

\bibliography{related_systems_references}

\clearpage
\onecolumn
\begin{appendices}
    \section{Proofs} \label{sec:proofs}
    
\begin{proof}[Proof of \cref{lemma:invertible}]
    We define $T=\set{\theta_1, \dots, \theta_{n}} \subset \Theta$. Let 
	\begin{align*}
		\mathcal{A}^\top_n \mathcal{S}_n = \begin{pmatrix}
			A_1^\top S_1  & \cdots & 0 \\
			\vdots & \ddots & \vdots \\
			0 & \cdots & A_n^\top S_n 
		\end{pmatrix}
	\end{align*}
    and note that $G_n$ can be expressed as:
	\begin{equation*}
		G_n = (\mathcal{A}^\top_n \mathcal{S}_n)^\top
		C_0(T,T) 
		\mathcal{A}^\top_n \mathcal{S}_n .
	\end{equation*}
    It follows from \citet[][Theorem 4.2.1]{golub_matrix_2013} that $G_n$ is positive definite (and therefore invertible) if $\mathcal{A}^\top_n \mathcal{S}_n$ has full column rank.
	Clearly,
    \begin{align}\label{eq:rank}
        \matrank(\mathcal{A}^\top_n \mathcal{S}_n) &= \sum_{i=1}^{n} \matrank(A_i^\top S_i) 
		= \sum_{i=1}^{n} \matrank(S_i^\top A_i) 
		= \sum_{i=1}^{n} \matrank(S_i^\top) 
		=  \sum_{i=1}^{n} m_i = M_n
    \end{align}
    which follows from the facts that $A_i$ is invertible and $S_i$ have full column rank by assumption.
	Therefore $\mathcal{A}^\top_n \mathcal{S}_n$ has full column rank, completing the proof.
\end{proof}

\begin{proof} [Proof of \cref{thm:rank(C_n-1)}]
	We know that 
	\begin{align*}
		C_{n-1}(\theta,\theta) = C_0(\theta,\theta) - K_{n-1}(\theta) G_{n-1}^\top K_{n-1}^\top(\theta) .
	\end{align*}
	Consider the matrix $M_1$
	\begin{align*}
		M_1 = \begin{bmatrix}
			G_{n-1} && K_{n-1}^\top(\theta) \\
			K_{n-1}(\theta) && C_0(\theta,\theta)
		\end{bmatrix}
	\end{align*}
    For a finite set $T \subset \Theta$, $T=\set{\theta_1, \dots, \theta_{n_T}}$, we define 
    \begin{align*}
        C_0(T,\theta) = \begin{bmatrix}
            C_0(\theta_1,\theta) && \cdots && C_0(\theta_{n_T},\theta) 
        \end{bmatrix}
    \end{align*}
    and similarly
    \begin{align*}
        C_0(T, T) = \begin{bmatrix}
            C_0(\theta_1,\theta_1) && \cdots && C_0(\theta_1,\theta_{n_T}) \\
            \vdots && \hdots && \vdots \\
            C_0(\theta_{n_T},\theta_1) && \cdots && C_0(\theta_{n_T},\theta_{n_T}) 
        \end{bmatrix}
    \end{align*}
    so that $C_0(T, \theta) \in \reals^{(dn_T) \times d}$ and $C_0(T, T) \in \reals^{(dn_T) \times (dn_T)}$.
    Next, let
    \begin{align*}
        \mathcal{A}_{n-1} &= \begin{bmatrix}
            A_1 && \cdots && 0 \\
            \vdots && \hdots && \vdots \\
            0 && \cdots && A_{n-1}
        \end{bmatrix} \\
        \mathcal{S}_{n-1} &= \begin{bmatrix}
            S_1 && \cdots && 0 \\
            \vdots && \hdots && \vdots \\
            0 && \cdots && S_{n-1}
        \end{bmatrix}.
    \end{align*} 
    Also define $T_{n-1} = \set{\theta_1, \dots, \theta_{n_T-1}}$, and $ \theta = \theta_{n_T}$. 
    Thus we have that 
    \begin{align*}
        K_{n-1}(\theta) &= \mathcal{S}_{n-1}^\top \mathcal{A}_{n-1} C_0(T_{n-1}, \theta)\\
        G_{n-1} &= \mathcal{S}_{n-1}^\top \mathcal{A}_{n-1} C_0(T_{n-1}, T_{n-1}) \mathcal{A}_{n-1}^\top \mathcal{S}_{n-1}.
    \end{align*}
    Next consider the Schur complement $M_1/C_0(\theta,\theta)$, defined as
    \begin{align*}
        M_1/C_0(\theta,\theta) &= G_{n-1} - K_{n-1}^\top(\theta) C_0(\theta,\theta)^{-1} K_{n-1}(\theta)\\
        &= \mathcal{S}_{n-1}^\top \mathcal{A}_{n-1} C_0(T_{n-1},\varTheta) \mathcal{A}_{n-1}^\top \mathcal{S}_{n-1} - \mathcal{S}_{n-1}^\top \mathcal{A}_{n-1} C_0(T_{n-1},\theta) C_0(\theta,\theta)^{-1} C_0(\theta, T_{n-1}) \mathcal{A}_{n-1}^\top \mathcal{S}_{n-1}\\
        &= M_1^1 M_1^2 (M_1^1)^\top
    \end{align*}
    where
    \begin{align*}
        M_1^1 &:= \mathcal{S}_{n-1}^\top \mathcal{A}_{n-1} \\
        M_1^2 &:= C_0({T_{n-1}, T_{n-1}})  - C_0(T_{n-1},\theta)  C_0(\theta,\theta)^{-1} C_0(\theta,T_{n-1})
    \end{align*}
    Next we note that $M_1^2$ is again a Schur complement, this time of the matrix $M_2$ given by
    \begin{align*}
        M_2 &= \begin{bmatrix}
            C_0(T_{n-1}, T_{n-1}) & C_0({T_{n-1},\theta_{n}}) \\
            C_0({\theta,T_{n-1}})  &  C_0(\theta,\theta)
        \end{bmatrix}
    \end{align*}
    so that $M_1^2 = M_2 / C_0(\theta, \theta)$.
    By assumption that $C_0$ is positive definite, both $C_0(\theta, \theta)$ and $M_2$ are positive definite. 
    Thus, from \citet[Proposition 2.1]{gallier2010schur}, the Schur complement $M_1^2$ is positive-definite. 
    We have seen in \cref{eq:rank} that $M_1^1$ has full row-rank. 
    Therefore, from \citet[Theorem 4.2.1]{golub_matrix_2013}, $M_1/C_0({\theta,\theta})$ is positive-definite, and thus non-singular.
   
    Since $G_{n-1}$, $C_0(\theta_n,\theta_n)$ and $M_1/C_0(\theta_n,\theta_n)$ are non-singular, from the matrix inversion lemma \citep[Corollary 2.8.8.]{bernstein_matrix_2009}, we have that $C_{n-1}(\theta_n,\theta_n) = M_1/G_{n-1}$ is also non-singular. Thus rank($C_{n-1}(\theta_n,\theta_n)$) = $d$.  
    
\end{proof}

\begin{lemma}\label{lem:iterCn}
    Given the posterior 
    covariance $C_{n-1}(\theta,\theta) A_n^T S_n$ computed without including the evaluated location $\theta$, posterior computed including the evaluating location can be computed iteratively as  
    \begin{align} 
        C_n(\theta,\theta) &= C_{n-1}(\theta,\theta) - C_{n-1}(\theta,\theta) A_n^T S_n  \left(S_n^T A_n C_{n-1}(\theta,\theta) A_n^T S_n \right)^{-1} S_n^T A_n C_{n-1}(\theta,\theta)  
    \end{align}
\end{lemma}
\begin{proof}
    We can see that 
    \begin{align*}
        C_n(\theta,\theta) &= C_0(\theta,\theta) \\
        &- \begin{bmatrix} K_{n-1}(\theta) & C_0(T_{n-1}, \theta_n)\end{bmatrix} \begin{bmatrix} G_{n-1} &K_{n-1}(\theta_n)^\top A_n^\top S_n \\ (K_{n-1}(\theta_n)^\top A_n^\top S_n)^\top  & S_n^\top A_{n} C_0(\theta_n,\theta_n) A_n^\top S_n \end{bmatrix}^{-1} \begin{bmatrix} K_{n-1}(\theta)^\top \\ C_0(T_{n-1}, \theta_n)\end{bmatrix} 
    \end{align*}
    The Schur complement of the matrix of the matrix to be inverted simplifies to $S_n^\top A_n C_{n-1}(\theta_n, \theta_n) A_n^\top S_n$. As seen in proof of \cref{lemma:invertible}, using \cref{thm:rank(C_n-1)}, this Schur complement is invertible. Thus the matrix inversion lemma \citep[Corollary 2.8.8.]{bernstein_matrix_2009} can be used here. Using this block matrix inversion result to simplify the above expression gives us \cref{lem:iterCn}. 
\end{proof}

\begin{proof} [Proof of \cref{thm:rank(C_n)}]
	From \cref{lem:iterCn}, we have that
	\begin{align*}
		C_n(\theta_n,\theta_n) &= C_{n-1}(\theta_n,\theta_n) - C_{n-1}(\theta_n,\theta_n) A_n^\top S_n  \left(S_n^\top A_n C_{n-1}(\theta_n,\theta_n) A_n^\top S_n \right)^{-1} S_n^\top A_n C_{n-1}(\theta_n,\theta_n)  
	\end{align*}
Consider 
\begin{align*}
	M_3 = \begin{bmatrix} S_n^\top A_n C_{n-1}(\theta_n,\theta_n)A_n^\top S_n &&  S_n^\top A_n C_{n-1}(\theta_n,\theta_n) \\
	C_{n-1}(\theta_n,\theta_n) A_n^\top S_n && C_{n-1}(\theta_n,\theta_n)
	\end{bmatrix}
\end{align*}
Clearly $C_n(\theta_n,\theta_n) = M_3 / S_n^\top A_n C_{n-1}(\theta_n, \theta_n) A_n^\top S_n$.
From Guttmans rank additivity theorem \citep[Fact 6.5.6]{bernstein_matrix_2009}, we have that
\begin{equation} \label{eq:rank_additivity}
	\text{rank}(M_3) = \text{rank}(S_n^\top A_n C_{n-1}(\theta_n,\theta_n)A_n^\top S_n) + \text{rank}(C_n(\theta_n,\theta_n)).
\end{equation}
We now calculate $\text{rank}(S_n^\top A_n C_{n-1}(\theta_n,\theta_n)A_n^\top S_n)$: since from \cref{thm:rank(C_n-1)}, we have that $C_{n-1}(\theta_n,\theta_n)$ has full rank and is thus invertible. Since $A_n$ is also invertible, from \citet[Proposition 2.6.3]{bernstein_matrix_2009} we have that
\begin{align*}
	\text{rank}(S_n^\top A_n C_{n-1}(\theta_n,\theta_n)A_n^\top S_n) = \text{rank}(S_n) = m_n.
\end{align*}
Next we calculate $\text{rank}(M_3)$. 
Clearly $M_3$ can be decomposed as follows
\begin{align*}
	M_3 = \underbrace{\begin{bmatrix}
		S_n^\top A \\ I_d
	\end{bmatrix}}_{M_3^1}
	\underbrace{\begin{bmatrix}
		C_{n-1} (\theta,\theta)
	\end{bmatrix}}_{M_3^2}
	\underbrace{\begin{bmatrix}
		A_n^\top S_n && I_d
	\end{bmatrix}}_{(M_3^1)^\top}.
\end{align*}  
Here $M_3^1$ has full column rank and $M_3^2$ has full rank, thus $\matrank(M_3^1 M_3^2) = \matrank(M_3^2) = d$ \citep[Proposition 2.6.3 and 2.6.1]{bernstein_matrix_2009}.
Similarly, $(M_3^1)^\top$ has full row rank, and thus $\matrank(M_3) = \matrank((M_3^1 M_3^2)M_3^{1\top}) = \matrank(M_3^1 M_3^2) = d$. 
Plugging these numbers into \cref{eq:rank_additivity} and rearranging we obtain
\begin{align*}
	\text{rank}(C_n(\theta_n,\theta_n)) = d - m_n
\end{align*}
as required.

For the null space we simply apply $A_n^\top S_n$ on the right to show this directly:
\begin{align*}
	C_n(\theta_n, \theta_n) &= C_0(\theta_n, \theta_n) - K_n(\theta_n) G_n^{-1} K_n^\top(\theta_n) \\
	C_n(\theta_n, \theta_n) A_n^\top S_n &= C_0(\theta_n, \theta_n) A_n^\top S_n - K_n(\theta_n) G_n^{-1} K_n^\top(\theta_n) A_n^\top S_n
\end{align*}
Now examining the rightmost matrix we see that:
\begin{align*}
    G_n^{-1} K_n^\top(\theta_n) A_n^\top S_n &= \begin{bmatrix}
		G_{n-1} && K_n(\theta_{n-1})^\top A_n^\top S_n \\ S_n^\top A_n K_n(\theta_{n-1}) &&  S_n^\top A_n C_0(\theta_n,\theta_n) A_n^\top S_n
	\end{bmatrix}^{-1} \begin{bmatrix}
		K_n(\theta_{n-1})^\top A_n^\top S_n \\ S_n^\top A_n C_0(\theta_n,\theta_n) A_n^\top S_n
	\end{bmatrix}\\
    &= \begin{bmatrix}
        0 \\ I_{m_n}
    \end{bmatrix}\numberthis \label{eq:G_n_identity}
\end{align*}
since the matrix to which the inverse applied is a column of $G_n$.
Therefore
\begin{align*}
        C_n(\theta_n, \theta_n) A_n^\top S_n &= C_0(\theta_n, \theta_n) A_n^\top S_n - \begin{bmatrix}
			K_n(\theta_{n-1}) && C_0(\theta_n,\theta_n) A_n^\top S_n
		\end{bmatrix}
		\begin{bmatrix}
			0 \\ I_{m_n}
		\end{bmatrix}
	\\
	&= C_0(\theta_n, \theta_n) A_n^\top S_n - C_0(\theta_n, \theta_n) A_n^\top S_n \\
	&= 0.
\end{align*}
Thus $\spanspace(A_n^\top S_n) \in \nullspace( C_n(\theta_n, \theta_n))$. Since \\ $\matrank(C_n(\theta_n, \theta_n)) = d - m_n$, $\matnullity (C_n(\theta_n, \theta_n)) = m_n$. Since $\matrank(A_n^\top S_n) = m_n$, it follows that $\spanspace(A_n^\top S_n)$ is the whole null space of $C_n(\theta_n, \theta_n)$, completing the proof.

\end{proof}

\begin{proof} [Proof of \cref{thm:nullC_n}]
    Again we can show this directly by applying $S_n^\top A_n$ to $\bar{x}_n(\theta_n)$.
    We have
	\begin{align*}
        S_n^\top A_n \bar{x}_n(\theta_n) &= S_n^\top A_n x_0(\theta_n) + S_n^\top A_n K_n(\theta_n) G_n^{-1} z_n \\
        &= S_n^\top A_n x_0(\theta) + 
        \begin{bmatrix}
            0 && I_{m_n}
        \end{bmatrix} 
        \begin{bmatrix}
            z_{n-1} \\
            y_n - S_n^\top A_n x_0(\theta_n)
        \end{bmatrix}
    \end{align*}
    from \cref{eq:G_n_identity}. Therefore
    \begin{align*}
        S_n^\top A_n \bar{x}_n(\theta_n) &= S_n^\top A_n x_0(\theta) + y_n - S_n^\top A_n x_0(\theta_n) \\
        &= S_n^\top A_n x^\star(\theta_n)
	\end{align*}
as required.
\end{proof}

\begin{proof} [Proof of \cref{thm:pCG_spsd_pc}]
    The \textsc{pCG} algorithm for positive-definite preconditioner is defined as in \cref{alg:pCG}.
	We prove here that $x_k \rightarrow x^*$ in this algorithm even when the preconditioner is positive-semi definite, given a specific structure of the initial guess $x_0$ as stated in the theorem.    
    
    First note that since $P$ is symmetric, it is diagonalisable, i.e.\ $P = U \Lambda U^\top$, where $U$ is a unitary matrix and $\Lambda$ is a diagonal matrix with diagonal entries as the eigenvalues of $P$. 
    As $P$ is a positive semi-definite matrix, we can write $\Lambda$  as
    \begin{align*}
        \Lambda = \begin{pmatrix}
            \Lambda_1 & 0 \\
            0 & 0
        \end{pmatrix}
    \end{align*}
    where $\Lambda_1$ is a diagonal matrix with positive entries, i.e., 
    \begin{align*}
        \Lambda_1 = \begin{pmatrix}
            \lambda_1 & &\\
            &\ddots \\
            &&\lambda_r
        \end{pmatrix}, \lambda_1 \geq \lambda_2 \dots \lambda_{r} \geq 0, r = \text{rank}(P) = \textup{dim}( \mathcal{R}(P) ).
    \end{align*} 
    Thus, $\Lambda_1$ is a symmetric positive definite matrix. 
    We now define partitions $U_1$ and $U_2$ of matrix $U$, i.e.,
    \begin{align*}
        U = \begin{bmatrix}
            U_1 & U_2
        \end{bmatrix}
    \end{align*}
    
    such that the columns of $U_1$ and $U_2$ are orthonormal bases of $\mathcal{R}(P)$ and $\mathcal{R}(P)^\perp$ respectively. We now re-express the linear system in terms of $U$:
    \begin{align*}
        \tilde{x} &= U^\top x \\
        \tilde{b} &= U^\top b \\
        \tilde{A} &= U^\top A U
    \end{align*}
    so that clearly $\tilde{A}\tilde{x} = \tilde{b}$.
    We now express the quantities used the \cref{alg:pCG} in terms of the projected system.
	Considering the initial steps, the initial residual is given by
    \begin{align*}
        \tilde{r}_0 &= U^\top r_0 
        =	\begin{bmatrix}
                \tilde{r}_0^{(1)} \\ \tilde{r}_0^{(2)}
            \end{bmatrix}
    \end{align*}
    while
    \begin{align*}
		\tilde{z}_0 &= U^\top z_0 
		= \begin{bmatrix}
            \Lambda_1 \tilde{r}_0^{(1)} \\ 0
        \end{bmatrix} 
    \end{align*}
    and finally
    \begin{align*}
        \tilde{p}_0 &=  U^\top p_0
        = \begin{bmatrix}
            \Lambda_1 \tilde{r}_0^{(1)} \\ 0
        \end{bmatrix} 
    \end{align*}
    Thus we see that in the basis defined by $U$, the first search direction modifies $\tilde{x}_0^{(1)}$ by taking a step in the direction $\Lambda_1 \tilde{r}_0^{(1)}$, but makes no change to $\tilde{x}_0^{(2)} = \tilde{x}^{\star (2)}$.

    Continuing to the iterates, proceeding inductively, 
    \begin{align*}
        \tilde{\alpha}_k &= \frac{\tilde{r}_k^\top \tilde{z}_k}{\tilde{r}_{k-1}^\top \tilde{z}_{k-1}} 
        = \frac{(\tilde{r}_k^{(1)})^\top \tilde{z}_k^{(1)}}{(\tilde{r}_{k-1}^{(1)})^\top \tilde{z}_{k-1}^{(1)}} 
    \end{align*}
    Note that both the numerator and denominator simplify as $\tilde{z}_k^{(2)} = 0$. Similarly,

    \begin{align*}
        \tilde{x}_{k+1} &= U^\top x_{k+1} 
        = \begin{bmatrix}
            \tilde{x}_{k}^{(1)} + \tilde{\alpha}_k \tilde{p_k}^{(1)} \\ \tilde{x}_k^{(2)}
        \end{bmatrix} 
    \end{align*}
    i.e. the second component of $\tilde{x}_{k+1}$ is unchanged, and
    \begin{align*}
        \tilde{r}_{k+1} &= U^\top r_{k+1} 
        = \begin{bmatrix}
            \tilde{r}_{k}^{(1)} -  \tilde{\alpha}_k  \tilde{A}^{(1)} \tilde{p}_k^{(1)} \\
            \tilde{\alpha}_k  U_2^\top A U_1 \tilde{p}_k^{(1)}
        \end{bmatrix}
    \end{align*}
    where the last line is due to $\tilde{p}_k^{(2)} = 0$.
    Next applying the preconditioner we have
    \begin{align*}
        \tilde{z}_{k+1} &=  U^\top P r_{k+1} 
        = \begin{bmatrix}
            \Lambda_1 r_{k+1}^{(1)} \\ 0
        \end{bmatrix}.
    \end{align*}
    Considering $\beta_k$ we have 
    \begin{align*}
		\tilde{\beta}_k &= \frac{\tilde{r}_{k+1}^\top \tilde{z}_{k+1}}{\tilde{z}_{k}^\top \tilde{z}_{k}} 
        = \frac{(r_{k+1}^{(1)})^\top z_{k+1}^{(1)}}{(r_k^{(1)})^\top z_k^{(1)}}
    \end{align*}
    and lastly the next search direction is given by
    \begin{align*}
        \tilde{p}_{k+1} &= U^\top p_{k+1} 
        = \begin{bmatrix}
            \tilde{z}_{k+1}^{(1)} + \tilde{\beta}_k \tilde{p}_{k}^{(1)} \\ 0
            \end{bmatrix}
    \end{align*}
    again maintaining the structure of the second component being zero. The norm of the error is given by
\begin{align*}
    \| \tilde{r}_k \| = (\tilde{r}_k^\top \tilde{r}_k)^{(1/2)} =  \| \tilde{r}_k^{(1)} \| + \| \tilde{r}_k^{(2)} \|
\end{align*}
We can see a clear structure in these quantities in terms of $\tilde{x}^{(1)}_k$ and $\tilde{x}^{(2)}_k$, and they can be separated. Thus we evaluate these two components separately. Only considering $\tilde{x}^{(1)}_k$ component, we have the following algorithm:
\begin{algorithmic}
	\State $x_0$: initial guess
	\State $\tilde{x}_0 = \begin{bmatrix}
		U_1^\top x_0 \\ U_2^\top x_0
	\end{bmatrix} = \begin{bmatrix}
		\tilde{x}^{(1)}_0 \\ \tilde{x}^{(2)}_0
	\end{bmatrix}$
	\State $\tilde{r}_0^{(1)} = \tilde{b}^{(1)} - U_1^\top A U_1 U_1^\top x_0 - U_1^\top A U_2 U_2^\top x_0 $ 
	\State $\tilde{r}_0^{(1)} =  \left( U_1^\top b - U_1^\top A U_2 U_2^\top x_0\right) - U_1^\top A U_1 U_1^\top x_0 $ 
	\State $\tilde{z}_0^{(1)} = \Lambda_1 \tilde{r}_0^{(1)}$
	\State $\tilde{p}_0^{(1)} = \tilde{z}_0^{(1)}$
	\State $k = 0$
	\While {$\| r_k \| > \epsilon $}
		\State $\tilde{\alpha}_k = \frac{\tilde{r}_k^{(1)\top} \tilde{z}_k^{(1)}}{\tilde{r}_{k-1}^{(1)\top} \tilde{z}_{k-1}^{(1)}}$
		\State $\tilde{x}_{k+1}^{(1)} = \tilde{x}_{k}^{(1)} + \tilde{\alpha}_k \tilde{p}_k^{(1)} $
		\State $\tilde{r}_{k+1}^{(1)} = \tilde{r}_k^{(1)} - \tilde{\alpha}_k^{(1)} \tilde{A} \tilde{p}_k^{(1)}$
		\State $\tilde{z}_{k+1}^{(1)} = \Lambda_1 \tilde{r}_{k+1}^{(1)}$
		\State $\tilde{\beta}_k^{(1)} = \frac{\tilde{r}_{k+1}^{(1)\top} \tilde{z}_{k+1}^{(1)}}{\tilde{r}_{k}^{(1)\top} \tilde{z}_{k}^{(1)}}$
		\State $\tilde{p}_{k+1}^{(1)} = \tilde{z}_{k+1}^{(1)} + \tilde{\beta}_k \tilde{p}_{k}^{(1)}$
	\EndWhile
\end{algorithmic}
Note that $U_1^\top A U$ is positive definite because A is positive definite and $U$ has full column rank (refer to \citet[Theorem 4.2.1]{golub_matrix_2013}). The above algorithm is equivalent to \cref{alg:pCG} for solving the system $U_1^\top A U_1 x^\star = \bar{b}$, (where $\bar{b} = U_1^\top b - U_1^\top A U_2 U_2^\top x_0$)  with positive-definite preconditioner $\Lambda_1$. Thus,
\begin{align*}
	\tilde{x}_k^{1} &\rightarrow (U_1^\top A U_1)^{-1} \bar{b}	\\
    &= U_1^\top A^{-1} U_1 U_1^\top c  - U_1^\top A^{-1} U_2 (U_2^\top A^{-1} U_2)^{-1} U_2^\top A^{-1} U_1 U_1^\top c
\end{align*}
Where $c = b -  A U_2 U_2^\top x_0 $. The $\tilde{x}^{(2)}_k$ component equal to $\tilde{x}_0^{(2)} = U_2^\top x_0 $ throughout the algorithm. We can see that 
\begin{align*}	
	x_k &= U \tilde{x}_k
	= \begin{bmatrix}
		U_1 & U_2
	\end{bmatrix}
	\begin{bmatrix}
		\tilde{x}_k^{(1)} \\ \tilde{x}_k^{(2)}
	\end{bmatrix} 
	= U_1 \tilde{x}_k^{(1)} + U_2 \tilde{x}_k^{(2)} 
\end{align*}    
With this, we get the following convergence of $x_k$ - 
\begin{align*}        
	x_k &\rightarrow  U_1 U_1^\top A^{-1} U_1 U_1^\top c - U_1 U_1^\top A^{-1} U_2 (U_2^\top A^{-1} U_2)^{-1} U_2^\top A^{-1} U_1 U_1^\top c + U_2 U_2^\top x_0 \\
    &= A^{-1}c- A^{-1} U_2 (U_2^\top A^{-1} U_2)^{-1} U_2^\top A^{-1}c + U_2 U_2^\top x_0 
\end{align*}
Now substituting the value of $c$,
\begin{align*}
    &= A^{-1}b  - A^{-1} U_2 (U_2^\top A^{-1} U_2)^{-1} \left( U_2^\top A^{-1} b  - U_2^\top x_0 \right) 
\end{align*}
Given the condition that the initial guess is correct in the null space of the preconditioner, i.e., $U_2^\top x_0 = U_2^\top A^{-1}b$, our preconditioned Conjugate Gradient algorithm converges to the solution of the linear system $Ax = b$.  

\end{proof}

    \section{Implementation Details} \label{sec:implementation}
    
\subsection{Practical Implementation} \label{sec:practical}

An algorithmic description of the proposed \ac{CompCG} method is given in \cref{alg:CompCG}. There are several implementation details which are discussed below.

\paragraph{Cholesky Updates}
To calculate the posterior mean and covariance we must invert $G_n$, which would have cost $\mathcal{O}(M_n^3)$ if implemented naively. 
However, since we are in a sequential data setting we can exploit Cholesky update formulae to reduce this to a cost of $\mathcal{O}(m_n^3)$. 
Let $L_{11}$ be the lower Cholesky factor of a symmetric positive definite matrix $A_{11}$. Then the Cholesky decomposition of the block symmetric positive definite matrix 
$$
A = \begin{pmatrix}
A_{11} & A_{12}\\
A_{12}^T & A_{22} 
\end{pmatrix}
$$
can be computed as 
$$
L = \begin{pmatrix}
L_{11} & 0\\
L_{12} & L_{22} 
\end{pmatrix}
$$ 
where $L_{12} =  (L_{11}^{-1}A_{12})^T $ and $L_{22}$ is the Cholesky factor of $A_{12} - L_{12}L_{12}^T$. 
Thus, to add a new training data point, adding a new $m_n \times m_n$ block to $G_{n-1}$ of size $M_{n-1}$ would require a Cholesky factor for a matrix of size $m_n$ (as compared to finding a Cholesky factor of size $M_n = \sum_{1}^{n}m_n$), a matrix-matrix multiplication of matrices of size $M_{n-1} \times m_n$,  and a triangular solve of size $M_{n-1}$ for the computation of the new $G_n^{-1}$.

\paragraph{Hyperparameter Optimisation}\label{sec:hypprm}
As with any \ac{GP} regression problem, performance of the \ac{CompCG} method will be strongly affected by the choice of hyperparameters.

In this work we do not explore this, and set the hyperparameters manually based on knowledge of the problem at hand. We explore the effect of this choice on our simulation study setting in \cref{sec:app_simulation}.
A more principled approach would be to optimise the marginal likelihood on a small sample from the parameter space (e.g.\ sampled at random according to some distribution believed to cover significant values of $\theta$, or using a small latin hypercube or sparse grid design).
However we leave development of these ideas for future work.

\subsubsection{Training Set Size} \label{sec:train_size}
While the dominant cost of \ac{CompCG} is $\mathcal{O}(m_n^3)$, if $M_n$ becomes large this will nevertheless begin to dominate the cost of the computation.
We therefore discuss several strategies for limiting the size of the training set, below.

\paragraph{Choosing the size of Search Direction matrix} \label{sec:SDsize}

A natural way to reduce $M_n$ is to limit the size of the search direction matrix, $m_n$.
This represents a trade-off between more accurate posterior mean and covariance, and computational efficiency. 

We elect to use an ad-hoc choice of $m_i = \alpha d$ for some $0 < \alpha < 1$.
For our experiments we set $\alpha = 0.2$, which seems to result in good performance.

\paragraph{Update Limitation} \label{sec:adddata}
We could also reduce the frequency of training set updates, rather than updating at every $i=1,\dots,n$.
We have several suggestions here:
\begin{enumerate}
	\item Update the training set when the number of CG iterations required to solve the linear system is high. A high number of CG iterations indicates that the posterior mean and covariance are not efficient as initial guess and preconditioner respectively, suggesting the model should be improved by incorporating more recent data.
	\item Update the training set every $j\textsuperscript{th}$ parameter visited, $j > 1$. This would work best if the parameters explored are near to each other, ensuring the training set does not have irrelevant points.
	\item Update the training when the new parameter $\theta_k$ is sufficiently ``far'' from the parameters in the current training set $T$. This ensures that the training data is updated when the new linear system is different from the already explored parameter space. The decision of how far is sufficiently far in the parameter space is problem specific, however, and the cost of computing this distance could be non-negligible.
	\item Optimally introduce search directions from $S_n$ rather than selecting in an ``all-or-nothing'' fashion.
	For example, one could select $\tilde{m}_i \ll m_i$ directions that maximally differ from one another (e.g.\ in terms of a Grassmann distance \citep{zhang_grassmannian_2018}), or using an SVD to optimally compress the basis (as seen in \cite{pfortner_computation-aware_2024}).
\end{enumerate}

Ultimately, the best approach depends on the problem at hand. For our case of hyperparameter optimisation in \cref{sec:optim}, as the parameter space is presumed to change as the optimiser proceeds, leading to a varying local parameter space, we use the number of CG iterations to decide on updating the training set. This decision to update the model is taken after solving the linear system, by adding an additional checkpoint after line 12 in \cref{alg:CompCG}.

\paragraph{Truncation}
In addition to reducing the number of training set updates, we could also periodically truncate the training set to retain only the most relevant information.
Resetting the model to remove all the data is a natural choice in many applications, e.g.~when an optimiser moves to explore a new portion of the parameter space, previous information may no longer be relevant. We exploit this in \cref{sec:optim}, where we reset the model when there is a big change in the computed NLL, indicating a change in the parameter space. 
One could also optimally compress the training set as a whole using similar strategies as discussed in point 4 under ``Update limitation'', above. 
Block Cholesky downdates can be also be used to efficiently remove information from $G_n$. While adding extra blocks to the Cholesky decomposition is rather straightforward, removing blocks is more extensive, and we leave this implementation for future work.

\subsection{Complexity Analysis}\label{sec:complexity}

We analyse the cost required to solve a new linear system of the size $d$ using $m_n$ search directions, when the trained regression model has size $M_{n-1} = \sum_{i=1}^{n-1} m_i$. The computations required for the companion regression model can be spilt into two: (i) updating the model, and (ii) calculating the posterior mean and covariance. 
We ignore the (possibly non-negligible) cost of computing the search directions as this varies based on the method used, and focus only on the cost of updating and evaluating the regression model, as this is overhead on top of the cost of CG.

\paragraph{(i) Updating the model} 
This requires matrix multiplications costing $\mathcal{O}(m_nd^2)$ and computing the Cholesky factor of $G_n$.
Since the factor of $G_{n-1}$ is available this can be updated as described above, costing $\mathcal{O}(M_{n-1}m_n^2 + M_{n-1}^2m_n + m_n^3)$.

\paragraph{(ii) Mean and covariance calculations} 
These require two triangular solves costing $\mathcal{O}(M_{n}^2d)$ (lines 7 and 8 in \cref{alg:CompCG}), and matrix multiplications to calculate $K_n$ and posteriors (in lines 5, 9 and 10 of \cref{alg:CompCG}) costing $\mathcal{O}(M_{n}d^2 + M_{n}d )$ resulting in the total cost of $\mathcal{O}(M_{n}^2d + M_{n}d + M_{n}d^2)$.

\paragraph{Summary}
In total, the companion regression model would cost $\mathcal{O}( m_n^3 + M_{n-1}^2 m_n  + M_{n}^2 d + m_n^2 M_{n-1} + M_{n}d^2 + m_nd^2 + M_{n}d )$. Assuming we keep the size of the training set in check using truncation (i.e.~ keeping $M_n = \mathcal{O}(m_n)$), ignoring the lower order terms, this simplifies to $\mathcal{O}(m_n^3 + m_{n}^2 d + m_{n}d^2)$.
In case of an ill-conditioned dense matrix, we argue that as long as the information $y_i$'s are of small size, i.e., $m_n \ll d$ and thus $M_{n-1} \ll d$, this cost is justified as the reduction in the number of CG iterations makes up for the additional cost involved in maintaining the companion regression model.

\begin{algorithm}
    \begin{algorithmic}[1]
		\Procedure{\textsc{pCG}}{$A$, $b$, initial guess $x_0$, preconditioner $P$, error tolerance $\epsilon$}
        \State $r_0 = b - Ax_0$
        \State $z_0 = Pr_0$
        \State $p_0 = z_0$
        \State $k = 0$
        \While {$\| r_k \| > \epsilon $}
            \State $\alpha_k = \frac{r_k^\top z_k}{r_{k-1}^\top z_{k-1}}$
            \State $x_{k+1} = x_{k} + \alpha_k p_k $
            \State $r_{k+1} = r_k - \alpha_k A p_k$
            \State $z_{k+1} = P r_{k+1}$
            \State $\beta_k = \frac{r_{k+1}^\top z_{k+1}}{r_{k}^\top z_{k}}$
            \State $p_{k+1} = z_{k+1} + \beta_k p_{k}$
        \EndWhile
		\State \Return $x_k$
	\EndProcedure
    \end{algorithmic}
    \caption{The preconditioned conjugate gradient solver} \label{alg:pCG}
\end{algorithm}

\begin{algorithm*} 
	\caption{Conjugate gradient solver with companion regression model} \label{alg:CompCG}
	\begin{algorithmic}[1]
		\Procedure{\textsc{CompCG}}{$\{A_n, b_n, \theta_n\}_{n=1}^{l}$, error tolerance: $\epsilon$, prior covariance: $C_0$, type of search direction: $SD$, size of search direction matrix: $m$}
			\For{$n = 1, 2, \cdots l$ }
			\State $S$ = \textsc{SearchDirection}($SD$, $m$, $A_n$, $b_n$, $\theta_n$)
			\State Update $T$, $\mathcal{A}^\top \mathcal{S}$, $z$ with $\theta_n, A_n$ and $S$
			\State $K = C_0(\theta_n, T) \cdot \mathcal{A}^\top \mathcal{S}$  
			\State $G_{chol}$ = \textsc{CholeskyUpdate}($G_{chol}$, $A_n$,$S$, $\theta_n$) \Comment{Block Cholesky update}
			\State $\beta_1$ = $G_{chol}^{-1} K^\top $  
			\State $\beta$ = $G_{chol}^{-1}\beta_1$
			\State $\bar{x} = \beta^\top \cdot z$   \Comment{Posterior predictive mean}
			\State $C = C_0(\theta_n, \theta_n) - \beta^\top \cdot K^\top$ \Comment{Posterior predictive covariance} 
			\State $x_n$ = \textsc{pCG}($A_n$, $b_n$, $x_0 = \bar{x}$, $P = C$, $\epsilon = \epsilon$)
			\State \Return $x_n$
			\EndFor
		\EndProcedure
	\end{algorithmic}
\end{algorithm*}

    \section{Additional Results} \label{sec:table}
    \subsection{Simulation Study} \label{sec:app_simulation}

\cref{fig:hp} shows CG iterations for varying lengthscale as a test to how an incorrectly chosen lengthscale would affect the performance of CompCG. 
Note that the theoretical results leading to \cref{thm:pCG_spsd_pc} do not depend on how this length-scale is chosen.
However, the additional preconditioning effect observed between \cref{fig:cg_iters,fig:cg_iters_nl} is likely to be significantly affected by a poor hyperparameter choice.
We observe that CompCG Subset and CompCG BayesCG are quite robust. However, CompCG BayesCG\_Id runs into precision issues when the lengthscale is high (10), leading to poor performance. We note here that CompCG BayesCG\_Id is the worst performing variant among the three CompCG variants, and this experiment reinforces our preference for CompCG Subset and CompCG BayesCG over CompCG BayesCG\_Id.

\begin{figure}
		\centering
		\includegraphics[width= \columnwidth]{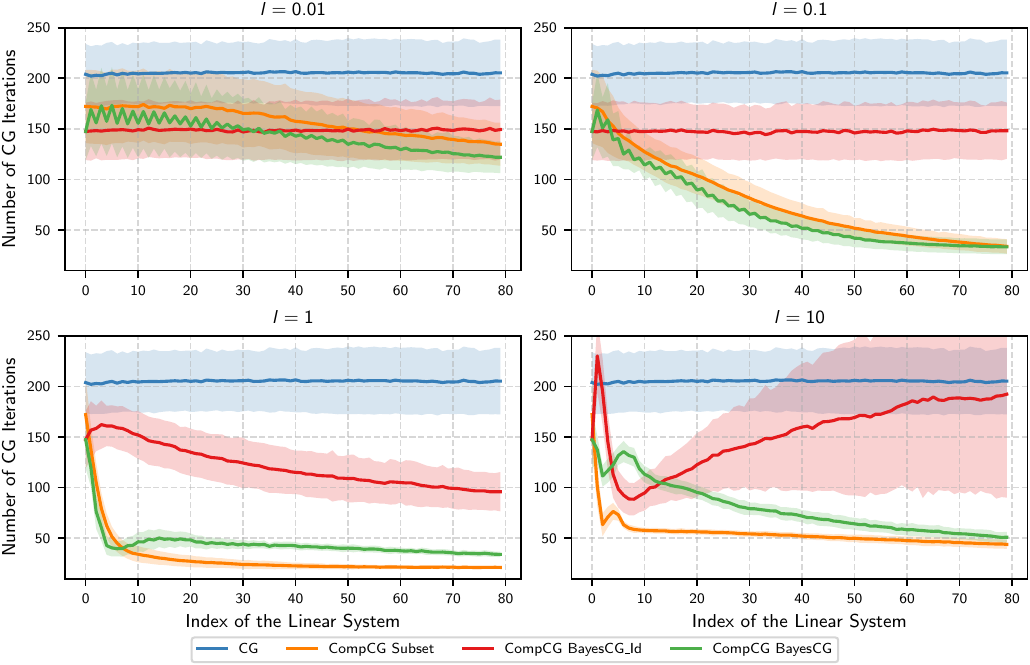}
		\caption{Number of CG iterations with differing lengthscale for the regression model} \label{fig:hp}
\end{figure}

The \ac{GP} mean from our regression model could be a good initial guess, in principle, independent of the proconditioner. But \cref{fig:cg_iters_ig} (of CG iterations using only initial guess computed from different search directions for the regression model) shows that just using the mean in this way does not significantly reduce the total number of iterations compared to CG. We hypothesise that this is because CG still ends up exploring the space in which the mean has already been identified because of the lack of deflation from the covariance preconditioner. Although this does not directly give us a motivation for extending the prior across the parameter space, there is still some evidence in our paper that the dependence structure imparts additional preconditioning benefit compared to just running some pre-iterations of BayesCG and using these to construct an initial guess and preconditioner, as seen in \cref{fig:cg_iters} and \cref{fig:cg_iters_nl}. This work does not include a theoretical justification for this, but we plan to explore this phenomenon more in future work.

\begin{figure}
\centering
\begin{minipage}{.48\textwidth}
  \centering
  \includegraphics[width= \columnwidth]{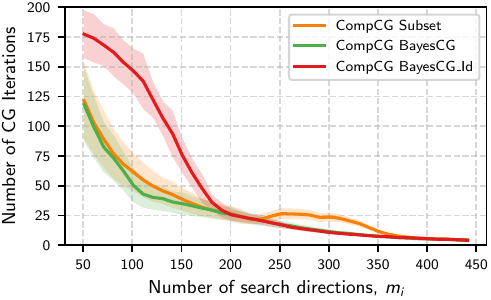}
		\caption{Number of CG iterations with differing number of search directions} \label{fig:cg_iters_ns}
\end{minipage}%
\hspace{0.02\textwidth}
\begin{minipage}{.48\textwidth}
  \centering
  \includegraphics[width= \columnwidth]{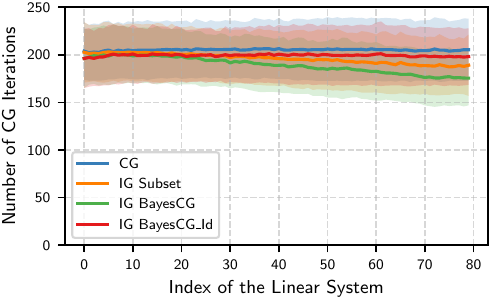}
		\caption{Number of CG iterations using only initial guess \textcolor{white}{filer-------}} \label{fig:cg_iters_ig}
\end{minipage}
\end{figure}

The amount of information used in training the regression model, determined by the number of search directions $m$ (as seen in \cref{eq:yi}), strongly influences the performance of the linear solver. \cref{fig:cg_iters_ns} shows the number of CG iterations required to solve the $40$\textsuperscript{th} linear system averaged over 20 runs, as a function of $m_i$. As expected, \cref{fig:cg_iters_ns} shows that the number of CG iterations decreasing with increasing $m_i$, as this results in better model training.

\subsection{GP Hyperparameter Optimisation Problem}
We present additional results of the GP hyperparameter optimisation problem as described in \cref{sec:optim} here.

\begin{table}
	
	\begin{tabular*}{\columnwidth}{l|lllllllllll}
		$d$ & 162  & 648 & 800 & 1035 & 1352 & 1800 & 2592 & 4050 & 7200 & 10368 & 16200 \\

        \grayline
		\multicolumn{11}{l}{\rule{0pt}{1em}Time for solving the linear systems (sec)} \\
		\grayline
		\rule{0pt}{1em}CG & 0.25  & 0.85 & 0.99 & 1.55 & 1.32 & 3.13 & 7.09 & 73.31 & 193.94 & 347.00 & 1671.25 \\
		\begin{tabular}{@{}l@{}}CompCG \\ Subset\end{tabular}  & \textbf{0.21}  & \textbf{0.43} & 0.51 & \textbf{0.79} & 1.89 & 1.47 & 3.09 & 14.19 & \textbf{48.09}  & 131.81 & \textbf{177.77}  \\
		\begin{tabular}{@{}l@{}}CompCG \\ BayesCG \end{tabular}  & 0.24 & 0.53 & \textbf{0.48} & 0.79 & \textbf{1.20} & \textbf{1.42} & \textbf{2.51} & \textbf{13.22} & 59.34  & \textbf{ 95.54}  & 272.27  \\

		\grayline
		\multicolumn{7}{l}{\rule{0pt}{1em}Average time for solving linear systems per optimiser iteration (sec)}  \\
		\grayline
		\rule{0pt}{1em}CG & 0.0031 & 0.0112 & 0.0118 & 0.0185 & 0.0220 & 0.0373 & 0.0844 & 0.8728 & 2.8521 & 6.6730 & 19.8959 \\
		\begin{tabular}{@{}l@{}}CompCG \\ Subset\end{tabular} &  
        0.0035 & 0.0083 & 0.0106 & \textbf{0.0116} & 0.0224 & 0.0263 & 0.0514 & \textbf{0.2729} & \textbf{0.7514} & 1.6476 & 4.0402  \\
		\begin{tabular}{@{}l@{}}CompCG \\ BayesCG\end{tabular}   & 
        \textbf{0.0029} & \textbf{0.0069} & \textbf{0.0086} & 0.0142 & \textbf{0.0215} & \textbf{0.0253} & \textbf{0.0419} & 0.2754 & 0.8726 & \textbf{1.4929} & \textbf{3.7815}  \\

		\grayline
		\multicolumn{11}{l}{\rule{0pt}{1em}Wall-time (sec)} \\
		\grayline
		\rule{0pt}{1em}CG & 0.82 & 4.78 & 4.99 & 7.50 & \textbf{8.75}  & 24.38 & 51.84 & 407.08 & 1475.93 & 4554.74 & 11241.33 \\   
		\begin{tabular}{@{}l@{}}CompCG \\ Subset \end{tabular} & \textbf{0.62 }& \textbf{2.80 }& \textbf{3.41} & 6.46 & 10.86 & \textbf{13.79} & 59.34 & 404.10 & \textbf{1202.75} & \textbf{2963.18} & 8440.28  \\
		\begin{tabular}{@{}l@{}}CompCG \\ BayesCG \end{tabular} &  0.81 & 2.90 & 3.84 & \textbf{6.09}& 14.50 & 23.33 & \textbf{49.27} & \textbf{297.70} & 1955.75 & 3135.40 & \textbf{7469.10}  \\

		\grayline
		\multicolumn{11}{l}{\rule{0pt}{1em}NLL} \\
		\grayline
		\rule{0pt}{1em}CG & 3264.2 & 7375.3 & 9164.8 & 11352.1 & 13477.0 & 16124.5 & 21192.6 & 29123.3 & 44313.8 & 59584.7 & 82928.9  \\
		\begin{tabular}{@{}l@{}}CompCG \\ Subset\end{tabular} &
        3261.2 & 7348.2 & 9111.6 & 11343.4 & 13464.0 & 16126.6 & 21184.6 & 29119.2 & 44289.1 & 59555.7 & 82909.6 \\
		\begin{tabular}{@{}l@{}}CompCG \\ BayesCG \end{tabular} &
        3267.1 & 7349.1 & 9122.4 & 11341.5 & 13480.7 & 16118.5 & 21165.3 & 29112.8 & 44296.4 & 59595.0 & 82958.4 \\
		\grayline
		\multicolumn{11}{l}{\rule{0pt}{1em}Total no. of CG iterations} \\
		\grayline
		\rule{0pt}{1em}CG & 5156 & 9663 & 12086 & 13572 & 10298 & 16794 & 21776 & 24143 & 24686 & 22237 & 47268 \\
		\begin{tabular}{@{}l@{}}CompCG \\ Subset\end{tabular}  & 2028 & 2193 & 1927  & 2307  & 3033  & 2472  & 2464  & 2811  & 3197  & 3785  & 3391  \\
		\begin{tabular}{@{}l@{}}CompCG \\ BayesCG \end{tabular}  & 
        2381 & 2796 & 2157  & 2362  & 2384  & 2481  & 2537  & 2293  & 3379  & 3236  & 4154  \\
		\grayline
		\multicolumn{11}{l}{\rule{0pt}{1em} No. of optimiser iterations} \\
		\grayline
		\rule{0pt}{1em}CG & 80 & 76 & 84 & 84 & 60 & 84 & 84 & 84 & 68 & 52 & 84 \\
		\begin{tabular}{@{}l@{}}CompCG \\ Subset\end{tabular}  & 60 & 52 & 48 & 68 & 84 & 56 & 60 & 52 & 64 & 80 & 44 \\
		\begin{tabular}{@{}l@{}}CompCG \\ BayesCG \end{tabular}    & 84 & 76 & 56 & 56 & 56 & 56 & 60 & 48 & 68 & 64 & 72 \\
        
	\end{tabular*}
	\caption{Additional results of NLL optimisation for the regression problem in Section 6.2 for varying $d$}\label{tbl:app_nll_times}
\end{table}

\end{appendices}
\end{document}